\documentclass[smallextended]{svjour3-arxiv}     % onecolumn (second format); required for SCW, see http://www.springer.com/economics/economic+theory/journal/355

\usepackage{amsmath}

\usepackage{amsthm}

\smartqed  % flush right qed marks, e.g. at end of proof
\usepackage{graphicx}
\usepackage{wrapfig}
\usepackage{url}
\newcommand\eat[1]{}
\usepackage{tikz}
\usepackage{booktabs}  %% nice tables
\usetikzlibrary{arrows}
\usepackage{tablefootnote}
\usepackage{slashbox}
	
        \journalname{}
\usepackage{array,xspace,multirow,hhline,graphicx,xcolor,tikz,colortbl,tabularx,amsmath,amssymb,amsfonts}
\usepackage[round]{natbib}
\usetikzlibrary{shapes}
\usepackage{algorithm,algorithmic}
\usepackage{mathrsfs}
\usepackage{enumitem}
\setenumerate[1]{label=(\emph{\roman{*}}),ref=(\emph{\roman{*}}),leftmargin=*}

\usepackage{txfonts}
\usepackage{eucal} % consistency with minstable

\usepackage{tikz}

\usepackage{mathtools}

	\newcommand{\Pref}[1][]{
		\ifthenelse{\equal{#1}{}}{\mathrel \succsim}{\mathop{R_{#1}}}
	}    
				                                      
	\newcommand{\sPref}[1][]{                  
		\ifthenelse{\equal{#1}{}}{\mathrel \succ}{\mathop{P_{#1}}}
	}                                          
	\newcommand{\Indiff}[1][]{                 
		\ifthenelse{\equal{#1}{}}{\mathrel \sim}{\mathop{\sim_{#1}}}
	}
	\newcommand{\prefset}[1][]{\ifthenelse{\equal{#1}{}}{\mathcal{R}}{\mathcal{R}_{#1}}}

\let\enumtemp=\enumerate
\def\enumerate{\enumtemp\itemsep 1pt}
\let\itemtemp=\itemize
\def\itemize{\itemtemp\itemsep 1pt}

\newcommand{\Omit}[1]{}

\begin{document}

\title{A Rule for Committee Selection\\ with Soft Diversity  Constraints} %Diversity

	\author{Haris Aziz}
	% \authorrunning{Aziz et al.}

	\institute{%
	  H. Aziz \at
	 Data61, CSIRO and UNSW,
	 	  Sydney 2052 , Australia \\
	 	  Tel.: +61-2-8306\,0490 \\
	 	  Fax: +61-2-8306\,0405 \\
	 \email{haris.aziz@unsw.edu.au}
}

	\newlength{\wordlength}
	\newcommand{\wordbox}[3][c]{\settowidth{\wordlength}{#3}\makebox[\wordlength][#1]{#2}}
	\newcommand{\mathwordbox}[3][c]{\settowidth{\wordlength}{$#3$}\makebox[\wordlength][#1]{$#2$}}
    		\renewcommand{\algorithmicrequire}{\wordbox[l]{\textbf{Input}:}{\textbf{Output}:}} 
    		 \renewcommand{\algorithmicensure}{\wordbox[l]{\textbf{Output}:}{\textbf{Output}:}}

\date{}


\maketitle

	\begin{abstract}
		Committee selection with diversity or distributional constraints is a ubiquitous problem. 
However, many of the formal approaches proposed so far have certain drawbacks including (1)
computationally intractability in general, and (2) inability to suggest a solution for certain instances where the hard constraints cannot be met. 
We propose a practical and polynomial-time algorithm for diverse committee selection that draws on the idea of using soft bounds and satisfies natural axioms.

	\end{abstract}



\section{Introduction}

Selecting a target number of candidates is a ubiquitous problem that occurs in faculty hiring, scholarship selection, corporate board election, and formation of representative bodies~\citep{Ratl06a}. In many of these settings, there are natural distributional constraints motivated for example by  diversity. For example, in several European countries, a wide-spread constraint is having a minimum 
percentage of females in corporate boards. In some school admission guidelines, there are quotas for less-advantaged groups. 

%
% If the candidates have an ordering in terms of desirability, then the problem is straightforward to solve. However often, committee selection involves distributional constraints motivated for example by  diversity.

Finding the best set of candidates subject to diversity constraints has also been formally studied in social choice. 
%These types of problems have recently been studied in  social choice. %However some recent approaches proposed so far are vulnerable to one of the following issues (1) the approach is computationally intractable in general  (2) the approach results in no real suggestion for certain instances. 
In several works, the problem of diverse committee selection is viewed as the problem with candidates having different (possibly multiple) types and then committee having distributions constraints on each of the types~(see e.g., \citep{BrPo90a, BGI+18a, CHV17a, Pott90a, SLSW93a}). %SLK+93a, 
There are a few drawbacks of approaches that use hard distributional constraints. The drawbacks include the following:  (1) there may be instances of the diverse committee selection problem that do not admit any feasible solution (for e.g., there simply may not be enough female applicants); (2) the hard constraints often make the problem of committee selection computationally intractable (for example, if we require that each type should have at least one representative, the problem of checking whether there exists a committee satisfying the requirement is NP-complete).
For approaches that are NP-hard, the lack of a simple polynomial-time algorithm makes these approaches impractical for large instances. Even for smaller instances, these approaches cannot be used without resorting to a computer. 
Finally, often real-life diversity guidelines need not be hard constraints but general rules of thumb to achieve procedural fairness.  
Finally, a hard constraint approach may also lead to a loss in efficiency.

Some other approaches consider distances between candidates or committees based on their type attributes and then view diversity not as a constraint but as an optimisation objective based on the distances~(see e.g., \citep{KGD93a,LaSk16a}). The approaches do not generally take into account the excellence of the candidates and the underlying problems are NP-hard. 
% \citet{LaSk16a} consider ideal compositions of the committee with the goal to find a committee minimizing distance from the ideal diversity distributions. 
% The latter approach does not take into account the excellence of the candidates and is also NP-hard.
Apart from imposing hard distributional constraints, another approach  that is often used in real-life committee selection to achieve diversity is to give bonus points or ranking boosts to candidates who are from under-represented groups.\footnote{The model where synergies or presence of diverse agents provide additional points to the committee has been considered in a general model by \citet{ITW18a}.} When these rules are imposed centrally, they may come across as arbitrary fixes to solving diversity issues. If the decision makers or voters internally take diversity issues into account while formulating an objective linear ranking, it puts a cognitive burden on the voters to mix diversity prioritisation with objective excellence estimation. 

%\citep{ABES17a}
% One way to circumvent hardness results due to hard constraints is to devise some greedy approaches that select candidates in a sequential manner. However, such approaches may not satisfy minimal and weak representation or fairness axioms.

In this paper, we consider the committee selection problem with distributional constraints and focus on the most common constraints whereby at least certain fraction of the candidates should satisfy a given type.\footnote{More general models also allow for expressing upper quotas. Often, the goal of the upper quotas can easily be met by setting lower quotas on the complement of the set of types.}
Our approach is to view the distributional constraints as \emph{soft} constraints which should be satisfied as much as possible.
We present a simple polynomial-time algorithm that simultaneously satisfies two axioms called \emph{type optimality} and \emph{justified envy-freeness}. 
The axiom justified envy-freeness is inspired by the matching market literature. 
The combination of the two axioms can be viewed as finding a committee that as close as possible to satisfying the hard distributional constraints and also selecting the best candidates.

% (1) the approaches are conceptually complex and not easy to implement by hand even for small instances

% \section{Related Work}
%
% \citet{BGI+18a}

%
% Celis, L. E.; Huang, L.; and Vishnoi, N. K. 2017. Group fairness in multiwinner voting. Technical Report arXiv:1710.10057 [cs.CY], arXiv.org.

% Potthoff, R. 1990. Use of linear programming for constrained approval voting. Interfaces 20(5):79–80.

% Straszak, A.; Libura, M.; Sikorski, J.; and Wagner, D. 1993. Computer-assisted constrained approval voting. Group Deci- sion and Negotiation 2(4):375–385.

\section{Setup}

The setting involves a set of candidates $C$, $\succsim$ a weak order over $C$, a set of types $T$, a matrix $\tau$ that specifies whether a candidate is of a certain type,  and a vector  $\underline{q}$ that specified the lower quotas bounds for each type. A diverse committee selection instance can be summarized as $(C,\succsim,T,\tau,\underline{q})$ where

% $T,\tau,\overline{q},\underline{q}$

\begin{itemize}
	\item $C=\{c_1,\ldots, c_{m}\}$ is the set of candidates. 
	\item The weak order $\succsim$ over $C$ is the priority order over the candidates. 
\item $T = \{t_1, t_2, ..., t_{\ell}\}$ is the set of types.
\item $\tau$ is a {matrix} consisting of each candidate's
type vector where
\begin{itemize}
\item $\tau_c:$ a type vector of candidate $c$ consisting of $1$'s and $0$'s 
\item $\tau_c^{t} = 1$ if $c$ belongs to type $t$ and $\tau_c^{t} = 0$ otherwise.
\end{itemize}
% \item $\overline{q}:$ a {matrix} consisting of type-specific upper bounds
% \begin{itemize}
% \item $\overline{q^t}:$ upper bound for type $t$.
% \end{itemize}
\item $\underline{q}$ is a {vector} consisting of all type-specific lower bounds. The value $\underline{q^t}$ denotes the lower bound for type $t$.
\end{itemize}
We will denote all types that a candidate $c$ belongs to by $\tau(c)$.
For $c,d\in C$, if $c\succsim d$ but $d\not\succsim c$, we will write the strict part of the relation as $c\succ d$.
Note that the model is powerful enough to capture the following kind of lower bounds: ``there should be at least $x$ members who are of one of the types from set $S\subset T$.'' In that case, one can create an `artificial'  type $t_S$ such that for any $c\in C$,  $\tau_c^{t_S}=1$ if $\tau_c^t=1$ for some $t\in S$.

For a committee $W\subset C$, we will denote by $\tau(W) = \sum_{c\in W}\tau_c.$ We will denote the number of candidate of type $t$ in $W$ by $\tau(W)(t)$. For some  committee $W\subset C$, if $\tau(W)(t)<\underline{q^t}$, we will say that type $t$ is under-represented in $W$.

The linear ranking over $C$ could be based on some objective measure that reflects the global quality of the candidate such as entrance examination scores. It could also be based on the aggregate scores based on some positional scoring voting done by voters who vote on the candidates~\citep{BrPo90a,BGI+18a}.

\begin{example}
	Consider the following instance. 
	\begin{itemize}
		\item $C=\{c_1,c_2,c_3,c_4\}$
		\item $c_1 \succ c_2 \succ c_3 \succ c_4$
		\item $k=2$
		\item $T=\{t_1,t_2,t_3, t_4\}$
		\item $\tau=\begin{pmatrix}
1&0&0&0\\
0&1&0&0\\
0&0&1&0\\
0&1&1&0\\
\end{pmatrix}$
		\item $\underline{q}=\begin{pmatrix}
0&1&2&1
\end{pmatrix}$
	\end{itemize}
	
	There is no feasible committee that can satisfy the hard constraints. 
	The committee $\{c_2,c_4\}$ satisfies all the constraints except the one corresponding to $t_4$. It is also more optimal than $\{c_3,c_4\}$. The latter does not satisfy justified envy-freeness. 
	
	\end{example}

\section{Axioms for Diverse Committee Selection}

% Solution:
%
% We first form a weak priority order over the candidates according to their objective quality.
% We then use a suitable lexicographic order over types to make the weak order a linear priority order. If two candidates have the same quality as well as the same types, we can break ties arbitrarily. We now implement the following algorithm.
%
% We first go over the types lexicographically. For the type under consideration, if the type is underrepresented, we select the candidate with the highest priority who has that type.
% If while selecting that type, we may increase an over subscribed type, then we select the one that has most type number going in te right direction.
%
%  If the type is not underrepresented any more, we move on to the next type.
%
%
% \hrule
%
%
%  First stage as above. But it may be the case that
%
% If both women and junior are undersubscribed, then a junior woman can replace any other member who has least number of attributes covered. We look for local feasible solution.
%
% Highest priority woman versus highest priority woman
%
% An agent can replace the least priority agent if the removal leads to a Pareto improvement in terms of targets.

We formalize some axiomatic properties that are desirable in our context. 
A committee $W$ satisfies type distribution $(x_1,\ldots, x_{|T|})$ if for each $i$, it has $x_i$ members of type $t_i$.

\begin{definition}[domination between type distributions]
A type distribution $x=(x_1,\ldots, x_{|T|})$ dominates another type $y=(y_1,\ldots, y_{|T|})$ if
\begin{enumerate}
	\item for each $t_i$ such that $y_i\geq \underline{q^i}$, we also have $x_i\geq \underline{q^i}$
	\item for each $t_i$ such that $y_i< \underline{q^i}$, either $x_i\geq \underline{q^i}$ or $|x_i - \underline{q^i}|\leq |y_i - \underline{q^i}|$
	\item there exists some $t_i$ such that either (1) $y_i< \underline{q^i}$ and $x_i\geq \underline{q^i}$ or (2) $x_i,y_i<\underline{q^i}$ and $|x_i - \underline{q^i}|\leq |y_i - \underline{q^i}|$.
\end{enumerate}
When $x$ dominates $y$, we denote it by $x>y$.
\end{definition}

Next, we observe the following lemma.

\begin{lemma}\label{lemma:transitive}
	If $\tau(W)>\tau(Y)$ and $\tau(Y)>\tau(Z)$, then $\tau(X)>\tau(Z)$.
	\end{lemma}

The argument is as follows. 
	If some type is not under-represented in $Z$, then it also not under-represented in $Y$. 
	By the same reasoning, if some type is not under-represented in $Y$, then it also not under-represented in $Z$. Hence it follows,  if some type is under-represented in $Z$, then it is at most as much under-represented in $Y$. Also there is one type that is less under-represented in $Y$ than in $Z$. Hence, type domination is a transitive relation.

Based on the notion of dominance between type distributions, we are now in a position to define type optimality of a committee. Roughly speaking, a committee is type optimal if it is `locally optimal' in terms of satisfying the distributional constraints.

\begin{definition}[Type optimal]
	A committee $W$ is type maximal if there exists no candidates $c\in W$ and $c'\in C\setminus W$ such that 
$\tau(W\setminus \{c\} \cup \{c'\})$ dominates $\tau(W)$.
	\end{definition}

	We note that type optimality is desirable in terms of distributional constraints but does not take into account the excellence of the candidates. Type optimality has been defined in a local sense based on swaps of candidates. If we
define it in a global sense by allowing swaps of subsets of candidates with subsets of candidates, then checking whether a given type distribution is optimal is NP-complete. 
	
	We now present an axiom that avoids scenarios where a candidate may feel that she deserves the place of a lesser ranked candidate. The axiom is adapted from the literature on stable matching with distributional constraints~\citep{KHIY17a,GKK+17a,KTY14a,EHYY14a,KaKo15a}.

	\begin{definition}[Justified envy-freeness]
		A committee $W$ satisfies justified envy-freeness if there is no candidate $c\notin W$ and $c'\in W$ such that $c\succ c'$ and there exists no type $t_i\in \tau(c')\setminus \tau(c)$ such that the number of candidates in $W$ of type $t_i$ is less than or equal to $\underline{q^i}$.  
		\end{definition}

We note that justified envy-freeness by itself can be trivially satisfied by a committee that selects the top $k$ ranked candidates. However, such a committee may not respect any of the distributional constraints. 

We have identified justified envy-freeness and type optimality 
as two axioms that individually may not have enough bite but together form a meaningful combination of axioms that lead to a desirable committee. Note that any committee that is optimal and satisfies the hard distributional constraints is both type optimal and satisfies justified envy-freeness. If it is not type optimal, then is not feasible. If it does not satisfy justified envy-freeness, then a swap of two candidates can increase the welfare of the committee without violating its hard constraints which means that the committee was not optimal subject to the constraints. In the next section, we present an algorithm that returns a committee satisfying both axioms.

\section{A Rule for Diverse Committee Selection}

We are in a position to present our algorithm (Algorithm~\ref{algo:diverse}) to find a diverse committee. 
In the first stage, the algorithm first checks if there is a type that is under-represented and then selects the highest priority candidate who satisfies such a type. 
The first stage is along similar lines as the Greedy Algorithm I of \citet{BGI+18a}.
If no type is under-represented, then the algorithm adds the required number of highest ranking candidates.
The second stage is geared towards obtaining a good type distribution. 
 If the committee does not satisfy type optimality, 
 candidates are exchanged with the goal to satisfy type optimality until it is satisfied. In the final while loop, the algorithm allows swaps of candidates if there is justified envy. The algorithm stops when the committee satisfies justified envy-freeness.

										\begin{algorithm}[h!]
											  \caption{Rule for finding a desirable committee satisfying soft distributional constraints on types}
											  \label{algo:diverse}
										% \footnotesize
										 % \small
			\normalsize
											\begin{algorithmic}
												\REQUIRE  $(C,\succsim,T,\tau,\underline{q})$.
												\ENSURE $W\subseteq C$ such that $|W|=k$
											\end{algorithmic}
											
											\begin{algorithmic}[1]
												\normalsize

	% \COMMENT{STAGE 1}
	\STATE $W\longleftarrow \emptyset$
	\WHILE{$|W|<k$ and there is some type $t$ that is underrepresented}
	\STATE Add a highest priority candidate of that type $t$ to $W$.
	\ENDWHILE	
	\IF{$|W|<k$}
	\STATE Select the highest ranked $k-|W|$ candidates from $C\setminus W$ into $W$.
	\ENDIF

	%
	% \COMMENT{STAGE 2}
	\label{step:endofstag1}
	% \WHILE{there can be a swap that results in  type-distribution improvement}
		\WHILE{there is a candidate $c\notin W$ and candidate $c'\in W$ such that $\tau(W\setminus \{c'\} \cup \{c\})$ dominates $\tau(W)$}
	\STATE $W\longleftarrow (W\setminus \{c'\})\cup \{c\}$
	\ENDWHILE
	
	% \COMMENT{STAGE 3}
		\label{step:endofstag2}
	\WHILE{there is a candidate $c\notin W$ who has justified envy for some candidate $c'\in W$}
	\STATE $W\longleftarrow (W\setminus \{c'\})\cup \{c\}$
	\ENDWHILE
 				
				\RETURN $W$
				\end{algorithmic}
				\end{algorithm}

				Let us illustrate the algorithm on our running example.
				\begin{example}
					
	Consider the following instance. 
	\begin{itemize}
		\item $C=\{c_1,c_2,c_3,c_4\}$
		\item $c_1 \succ c_2 \succ c_3 \succ c_4$
		\item $k=2$
		\item $T=\{t_1,t_2,t_3, t_4\}$
		\item $\tau=\begin{pmatrix}
1&0&0&0\\
0&1&0&0\\
0&0&1&0\\
0&1&1&0\\
\end{pmatrix}$
		\item $\underline{q}=\begin{pmatrix}
0&1&2&1
\end{pmatrix}$
	\end{itemize}

The algorithm first considers a type $t_2$ that is under-represented and selects $c_2$. It then considers type $t_3$ that is under-represented and selects $c_3$. At this point the type distribution is not optimal and can be improved 
by the exchange of $c_2$ with $c_4$ so that $W=\{c_3,c_4\}.$ At this point the committee is type optimal and also satisfies justified envy-freeness.
 					\end{example}

				\begin{proposition}
					
Algorithm 1 returns in time $O(|C|^2)$ a committee that satisfies justified envy-freeness.
					\end{proposition}
					\begin{proof}
Note that by Step~\ref{step:endofstag1}, we already have a committee of size $k$. The first while loop takes time at most $O(k|C|)$.
In the second while loop, the type distribution keeps improving since the type domination relation is transitive (Lemma~\ref{lemma:transitive}), but there can be at most $|C|$ such improvements. Finally, in the last while loop, with each swap of candidates, a candidate is replaced by a candidate with a higher ranking. This can happen at most $O(|C|^2)$ times.  						
						\end{proof}

		Although Algorithm 1 finds a type optimal committee by Step~\ref{step:endofstag2}, the committee undergoes further changes in the final while loop. We now argue that the returned committee still satisfies type optimality. 			
									\begin{proposition}
					
					Algorithm 1 returns a committee satisfying type optimality. 
										\end{proposition}
										\begin{proof}
											
		In the second while loop, we start from a committee of size $k$. By the end of Step~\ref{step:endofstag2}, $W$ is a committee that is type optimal. In the final while loop, we implement swaps if there is a candidate $c\notin W$ who has justified envy for some candidate $c'\in W$. Suppose we swap $c'$ with $c$. We claim that $W\setminus \{c'\}\cup \{c\}$ is also type optimal. Since $c$ had a justified envy against $c'$, by definition of justified envy, 
there exists no type $t_i\in \tau(c')\setminus \tau(c)$ such that the number of candidates in $W$ of type $t_i$ is less than or equal to  $\underline{q^i}$.	If there is a type $t_i$
such that $\tau(W)(i)\geq\underline{q^i}$, then $\tau(W\setminus \{c'\}\cup \{c\})(i)\geq\underline{q^i}$.
In words, if a type not under-represented in $W$, it is not under-represented in $W\setminus \{c'\}\cup \{c\}$. 
If there is a type $t_i$
such that $\tau(W)(i)<\underline{q^i}$, then 
$\tau(W)(i)\leq \tau(W\setminus \{c'\}\cup \{c\})(i) \leq \underline{q^i}$. In words, if a type is under-represented in $W$, it is at most as under-represented in $W\setminus \{c'\}\cup \{c\}$. Thus $\tau(W\setminus \{c'\}\cup \{c\})$ is as good as $\tau(W)$. Since $W$ was type optimal, $W\setminus \{c'\}\cup \{c\}$ is type optimal as well. 
 \end{proof}
						
						We have shown that our algorithm simultaneously satisfies both type optimality and justified envy-freeness. Since these are two key properties satisfied by optimal committees satisfying hard constraints, our algorithm provides a computationally easy and principled rule to find desirable committees that almost satisfy the distributional constraints.
Although each of the stages of the algorithm is based on natural greedy or local swap-based approaches, our formulation of reasonable axioms provide a guiding force towards a committee that is desirable. It will be interesting to see if other desirable axioms can be simultaneously satisfied in the conjunction with the ones which our algorithm satisfies. 
%We also remark that the outcome $W$ is Pareto optimal with respect to the responsive set extension among all other committees that satisfy the same type distribution.  						
						
				Our general algorithm can have more precise specifications that prioritise certain types in a lexicographical manner or implement swaps according to some pre-determined pattern. The $\succsim$ ranking order can be derived by using some social welfare function for a set of voters voting on the quality of the  candidates.  				
										
													% 	\begin{proposition}
				% Algorithm returns a committee satisfying justified envy-freeness. 					\end{proposition}
				%

% In controlled school choice with soft bounds, school districts adopt a dynamic priority structure: giving highest priority to candidate types who have not filled their floors; medium priority to candidate types who have filled their floors, but not filled their ceilings; and lowest priority to candidate types who have filled their ceilings.

\paragraph{Acknowledgements}

The author is supported by a Julius Career Award. He thanks  Steven Brams and Ayumi Igarashi for useful comments and pointers.

\bibliographystyle{plainnat}
 % \bibliography{../../pamas/abb,../../pamas/group,../../pamas/brandt,../../pamas/aziz}
% %
% \bibliography{../../pamas/abb,../../pamas/brandt,../../pamas/group,../../pamas/aziz,../../pamas/haris_master}

\end{document}